%% file: nips_2018.tex
\documentclass{article}

\PassOptionsToPackage{numbers, compress}{natbib}

\usepackage[final]{nips_2018}

\usepackage[utf8]{inputenc} 
\usepackage[T1]{fontenc}    
\usepackage{hyperref}       
\usepackage{url}            
\usepackage{booktabs}       
\usepackage{amsfonts}       
\usepackage{nicefrac}       
\usepackage{microtype}      

\usepackage{comment}
\usepackage{subfig}
\usepackage{epsf}
\usepackage{latexsym}
\usepackage{amssymb}
\usepackage{amsmath}
\usepackage{bbm}
\usepackage[usenames,dvipsnames]{xcolor}
\usepackage{multicol}

\usepackage{stmaryrd}

\usepackage{algorithm}
\usepackage[noend]{algpseudocode}
\usepackage{amsopn}
\usepackage{graphics}
\usepackage{epsfig}
\usepackage{color}
\usepackage{enumerate}
\usepackage{graphicx}
\usepackage{mathtools}

\newtheorem{theorem}{Theorem}
\newtheorem{lemma}[theorem]{Lemma}
\newtheorem{proposition}[theorem]{Proposition}
\newtheorem{definition}[theorem]{Definition}

\newtheorem{claim}[theorem]{Claim}

\usepackage{prettyref}

\newcommand{\savehyperref}[2]{\texorpdfstring{\hyperref[#1]{#2}}{#2}}

\newrefformat{eq}{\savehyperref{#1}{\textup{(\ref*{#1})}}}
\newrefformat{lem}{\savehyperref{#1}{Lemma~\ref*{#1}}}
\newrefformat{def}{\savehyperref{#1}{Definition~\ref*{#1}}}
\newrefformat{thm}{\savehyperref{#1}{Theorem~\ref*{#1}}}
\newrefformat{cor}{\savehyperref{#1}{Corollary~\ref*{#1}}}
\newrefformat{corr}{\savehyperref{#1}{Corollary~\ref*{#1}}}
\newrefformat{cha}{\savehyperref{#1}{Chapter~\ref*{#1}}}
\newrefformat{sec}{\savehyperref{#1}{Section~\ref*{#1}}}
\newrefformat{app}{\savehyperref{#1}{Appendix~\ref*{#1}}}
\newrefformat{tab}{\savehyperref{#1}{Table~\ref*{#1}}}
\newrefformat{fig}{\savehyperref{#1}{Figure~\ref*{#1}}}
\newrefformat{hyp}{\savehyperref{#1}{Hypothesis~\ref*{#1}}}
\newrefformat{alg}{\savehyperref{#1}{Algorithm~\ref*{#1}}}
\newrefformat{item}{\savehyperref{#1}{Item~\ref*{#1}}}
\newrefformat{step}{\savehyperref{#1}{step~\ref*{#1}}}
\newrefformat{conj}{\savehyperref{#1}{Conjecture~\ref*{#1}}}
\newrefformat{fact}{\savehyperref{#1}{Fact~\ref*{#1}}}
\newrefformat{prop}{\savehyperref{#1}{Proposition~\ref*{#1}}}
\newrefformat{claim}{\savehyperref{#1}{Claim~\ref*{#1}}}
\newrefformat{algo}{\savehyperref{#1}{Algorithm~\ref*{#1}}}
\newrefformat{rem}{\savehyperref{#1}{Remark~\ref*{#1}}}
\newrefformat{line}{\savehyperref{#1}{Line~\ref*{#1}}}
\let\pref=\prettyref

\newcommand{\sq}{\hbox{\rlap{$\sqcap$}$\sqcup$}}
\newcommand{\qed}{\hspace*{\fill}\sq}
\newenvironment{proof}{\noindent {\it Proof.}\ }{\qed\par\vskip 4mm\par}

\renewcommand{\Pr}{\mathbf{Pr}}

\DeclareMathOperator*{\E}{{\mathop{\mathbb{E}}}}
\newcommand{\utest}{\textsc{Test}}
\newcommand{\ntest}{\textsc{WeakTest}}
\newcommand{\BMW}{\textsc{BasicMW}}
\newcommand{\IMW}{\textsc{MWeights}}
\newcommand{\eps}{{\epsilon}}

\newcommand{\err}{\mathrm{err}}

\newcommand{\tT}{\tilde{{T}}}

\newcommand{\tQ}{\tilde{Q}}

\newcommand{\calA}{\mathcal{A}}

\newcommand{\calE}{\mathcal{E}}
\newcommand{\calF}{\mathcal{F}}
\newcommand{\calG}{\mathcal{G}}
\newcommand{\calJ}{\mathcal{J}}
\newcommand{\calL}{\mathcal{L}}
\newcommand{\calX}{\mathcal{X}}
\newcommand{\calY}{\mathcal{Y}}
\newcommand{\bbS}{\mathbb{S}}

\title{Tight Bounds for Collaborative PAC Learning via Multiplicative Weights}

\author{
  Jiecao Chen\\
  Computer Science Department\\
  Indiana University at Bloomington\\
  \texttt{jiecchen@iu.edu}\\
  \And
  Qin Zhang\\
  Computer Science Department\\
  Indiana University at Bloomington\\
  \texttt{qzhangcs@indiana.edu}\\
  \And
  Yuan Zhou\\
  Computer Science Department\\
  Indiana University at Bloomington\\
  and\\
  Department of Industrial and Enterprise Systems Engineering\\ 
  University of Illinois at Urbana-Champaign\\
  \texttt{yuanz@illinois.edu}
}

\begin{document}

\maketitle

\begin{abstract}
  We study the collaborative PAC learning problem recently proposed in Blum  et al.~\cite{BHPQ17}, in which we have $k$ players and they want to learn a target function collaboratively, such that the learned function approximates the target function well on all players' distributions simultaneously. The quality of the collaborative learning algorithm is measured by the ratio between the sample complexity of the algorithm and that of the learning algorithm for a single distribution (called the overhead).  We obtain a collaborative learning algorithm with overhead $O(\ln k)$, improving the one with overhead $O(\ln^2 k)$ in \cite{BHPQ17}.  We also show that an $\Omega(\ln k)$ overhead is inevitable when $k$ is polynomial bounded by the VC dimension of the hypothesis class.  Finally, our experimental study has demonstrated the superiority of our algorithm compared with the one in Blum  et al.~\cite{BHPQ17} on real-world datasets.
\end{abstract}

\input{intro}

\input{up-basic}

\input{up-improved}

\input{lb}

\input{exp}

\input{conclusion}

\newpage

\subsubsection*{Acknowledgments}
Jiecao Chen and Qin Zhang are supported in part by NSF CCF-1525024 and IIS-1633215. Part of the work was done when Yuan Zhou was visiting the Shanghai University of Finance and Economics.

\bibliography{paper}
\bibliographystyle{abbrv}

\newpage
\appendix
\input{appendix}

\end{document}

%% file: intro.tex
\section{Introduction}
\label{sec:intro}

In this paper we study the collaborative PAC learning problem recently proposed in Blum et al.~\cite{BHPQ17}.  In this problem we have an instance space $\calX$, a label space $\calY$, and an unknown target function $f^* : \calX \to \calY$ chosen from the hypothesis class $\calF$. We have $k$ players with distributions $D_1, D_2, \ldots, D_k$ labeled by the target function $f^*$.   Our goal is to \emph{probably approximately correct (PAC)} learn the target function $f^*$ for {\em every} distribution $D_i$. That is, for any given parameters $\eps, \delta > 0$, we need to return a function $f$ so that with probability $1 - \delta$, $f$ agrees with the target $f^*$ on instances of at least $1 - \eps$ probability mass in $D_i$ for every player $i$.

As a motivating example, consider a scenario of personalized medicine where a pharmaceutical company wants to obtain a prediction model for dose-response relationship of a certain drug based on the genomic profiles of individual patients. While existing machine learning methods are efficient to learn the model with good accuracy for the whole population, for fairness consideration, it is also desirable to ensure the model accuracies among demographic subgroups, e.g. defined by gender, ethnicity, age, social-economic status and etc., where each of them is associated with a label distribution. 

We will be interested in the ratio between the sample complexity required by the best collaborative learning algorithm and that of the learning algorithm for a single distribution, which is called the \emph{overhead} ratio. A na\"ive approach for collaborative learning is to allocate a uniform sample budget for each player distribution, and learn the model using all collected samples. In this method, the players do minimal collaboration with each other and it leads to an $\Omega(k)$ overhead for many hypothesis classes (which is particularly true for the classes with fixed VC dimension -- the ones we will focus on in this paper).  In this paper we aim to develop a collaborative learning algorithm with the optimal overhead ratio.


\paragraph{Our Results.} We will focus on the hypothesis class $\calF = \{f: \calX \to \calY\}$ with VC dimension $d$. For every $\eps, \delta > 0$, let $\bbS_{\eps, \delta}$ be the sample complexity needed to $(\eps, \delta)$-PAC learn the class $\calF$. It is known  that there exists an $(\eps, \delta)$-PAC learning algorithm $\calL_{\eps, \delta, \calF}$ with  $\bbS_{\eps, \delta} = O\left(\frac{1}{\eps} \left(d  + \ln \delta^{-1}\right)\right)$~\cite{Hanneke16}. We remark that we will use the algorithm $\calL$ as a blackbox, and therefore our algorithms can be easily extended to  other hypothesis classes given their single-distribution learning algorithms.

Given a function $g$ and a set of samples $T$, let $\err_{T}(g) = \Pr_{(x, y) \in T}[g(x) \neq y]$ be the \emph{error} of $g$ on $T$. Given a distribution $D$ over $\calX \times \calY$, define $\err_{D}(g)  = \Pr_{(x, y) \sim D}[g(x) \neq y]$ to be the \emph{error} of $g$ on $D$. The $(\eps, \delta)$-PAC $k$-player collaborative learning problem can be rephrased as follows: For player distributions $D_1, D_2, \dots, D_k$ and a target function $f^* \in \calF$, our goal is to learn a function $g : \calX \to \calY$ so that 
$\Pr[\forall i = 1, 2, \dots k, \err_{D_i}(f^*, g) \leq \eps] \geq 1 - \delta$.
Here we allow the learning algorithm to be \emph{improper}, that is, the learned function $g$ does not have to be a member of $\calF$.

Blum et al.\  \cite{BHPQ17} showed an algorithm with sample complexity $O\left(\frac{\ln^2 k}{\eps}\left((d + k) \ln \eps^{-1} + k \ln \delta ^{-1}\right)\right)$. When $k = O(d)$, this leads to an overhead ratio of $O(\ln^2 k)$ (assuming $\eps$, $\delta$ are constants). In this paper we propose an algorithm with sample complexity $O\left(\frac{(\ln k + \ln \delta^{-1})(d + k)}{\eps}\right)$ (\pref{thm:main-2}), which gives an overhead ratio of $O(\ln k)$ when $k = O(d)$ and for constant $\delta$, matching the $\Omega(\ln k)$ lower bound proved in Blum et al.\  \cite{BHPQ17}. 

Similarly to the algorithm in Blum et al.\  \cite{BHPQ17}, our algorithm runs in rounds and return the plurality of the functions computed in each round as the learned function $g$. In each round, the algorithm adaptively decides the number of samples to be taken from each player distribution, and calls $\calL$ to learn a function. While the algorithm in Blum et al.\  \cite{BHPQ17} uses a grouping idea and evenly takes samples from the distribution in each group, our algorithm adopts the multiplicative weight method. In our algorithm, each player distribution is associated with a weight which helps to direct the algorithm to distribute the sample budget among all player distributions. After each round, the weight for a player distribution increases if the function learned in the round is not accurate on the distribution, letting the algorithm pay more attention to it in the future rounds. We will first present a direct application of the multiplicative weight method which leads to a slightly worse sample complexity bound (\pref{thm:main-1}), and then prove \pref{thm:main-2} with more refined algorithmic ideas.

On the lower bound side, the lower bound result in Blum et al.\  \cite{BHPQ17} is only for the special case when $k = d$. We extend their result to every $k$ and $d$. In particular, we show that the sample complexity for collaborative learning has to be $\Omega(\max\{d\ln k, k \ln d\}/\eps)$ for constant $\delta$ (\pref{thm:lb}). Therefore, the sample complexity of our algorithm is optimal when $k = d^{O(1)}$. \footnote{We note that this is a stronger statement than the earlier one on the ``the optimal overhead ratio of $O(\ln k)$ for $k = O(d)$'' in several aspects. First, the showing the optimal overhead ratio only needs a minimax lower bound; while in the latter statement we claim the optimal sample complexity for every $k$ and $d$ in the range. Second, our latter statement works for a much wider parameter range for $k$ and $d$. } 

Finally, we have implemented our algorithms and compared with the one in Blum et al.\  \cite{BHPQ17} and the na\"ive method on several real-world datasets. Our experimental results demonstrate the superiority of our algorithm in terms of the sample complexity.

\paragraph{Related Work.}
As mentioned, collaborative PAC learning was first studied in Blum et al.\ \cite{BHPQ17}.  Besides the problem of learning one hypothesis that is good for all players' distributions (called the {\em centralized collaborative learning} in \cite{BHPQ17}), the authors also studied the case in which we can use different hypotheses for different distributions (called {\em personalized collaborative learning}).  For the personalized version they obtained an $O(\ln k)$ overhead in sample complexity.  Our results show that we can obtain the same overhead for the (more difficult) centralized version.  In a concurrent work \cite{nguyen2018improved}, the authors showed the similar results as in our paper.

Both our algorithms and Adaboost \cite{freund1997decision} use the multiplicative weights method. While Adaboost places weights on the samples in the prefixed training set, our algorithms place weights on the distributions of data points, and adaptively acquire new samples to achieve better accuracy. Another important feature of our improved algorithm is that it tolerates a few ``failed rounds'' in the multiplicative weights method, which requires more efforts in its analysis and is crucial to shaving the extra $\ln k$ factor when $k = \Theta(d)$.

Balcan et al.\ \cite{BBFM12} studied the problem of finding a hypothesis that approximates the target function well on the joint mixture of $k$ distributions of $k$ players. They focused on minimizing the communication between the players, and allow players to exchange not only samples but also hypothesis and other information.  Daume et al.~\cite{DPSV12a,DPSV12b} studied the problem of computing linear separators in a similar distributed communication model.  The communication complexity of distributed learning has also been studied for a number of other problems, including principal component analysis~\cite{LBKW14}, clustering~\cite{BEL13,GYZ17}, multi-task learning~\cite{WKS16}, etc.

Another related direction of research is the multi-source domain adaption problem \cite{MMR08}, where we have $k$ distributions, and a hypothesis with error at most $\eps$ on each of the $k$ distributions. The task is to combine the $k$ hypotheses to a single one which has error at most $k \eps$ on any mixture of the $k$ distribution. This problem is different from our setting in that we want to learn the ``global'' hypothesis from scratch instead of combine the existing ones.



%% file: up-basic.tex
\section{The Basic Algorithm}
\label{sec:basic}

In this section we propose an algorithm for collaborative learning using the multiplicative weight method.  The algorithm is described in \pref{alg:main-1}, using \pref{alg:test} as a subroutine.  

\begin{multicols}{2}

We briefly describe \pref{alg:main-1} in words.  We start by giving a unit weight to each of the $k$ player.  The algorithm runs in $T = O(\ln k)$ rounds, and players' weights will change at  each round.  At round $t$, we take a set of samples $S^{(t)}$ from the average distribution of the $k$ players weighted by their weights. We then learn a classifier $g^{(t)}$ for samples in $S^{(t)}$, and test for each player $i$ whether $g^{(t)}$ agrees with the target function $f^*$ with probability mass at least $1 - \eps/6$ on distribution $D_i$.  If yes then we keep the weight of the $i$-th player; otherwise we multiply its weight by a factor of $2$, so that $D_i$ will attract more attention in the future learning process.  Finally, we return a classifier $g$ which takes the plurality vote\footnote{I.e.\ the most frequent value, where ties broken arbitrarily.} of the $T$ classifiers $g^{(0)}, g^{(1)}, \ldots, g^{(T-1)}$ that we have constructed. We note that we make no effort to optimize the constants in the algorithms and their theoretical analysis; while in the experiment section, we will tune the constants for better empirical performance.

The following lemma shows that \utest\ returns, with high probability, the desired set of players where $g$ is an accurate hypothesis for its own distribution. We say a call to \utest\ \emph{successful} if its returning set has the properties described in \pref{lem:utest}. The omitted proofs in this section can be found in \pref{app:basic}.

\begin{algorithm}[H]  
\caption{{\BMW}}
\label{alg:main-1}
\begin{algorithmic}[1]  
\State Let the initial weight $w_i^{(0)} \leftarrow 1$ for each player $i \in \{1, 2, \dots, k\}$.
\State Let $T \leftarrow 10 \ln k$.   \label{line:alg-main-1-1}
\For {$t \leftarrow 0 \textbf{~to~} T - 1$}
	\State Let $p^{(t)}(i) \leftarrow \frac{w_i^{(t)}}{\sum_{i=1}^{k} w_i^{(t)}}$ for each $i \in \{1, 2, \dots, k\}$ so that $p^{(t)}(\cdot)$ defines a probability distribution.
	\State Let $D^{(t)} \leftarrow \sum_{i = 1}^{K} p^{(t)}(i) D_i$.
	\State Let  $S^{(t)}$ be a set of $\bbS_{\frac{\eps}{120}, \frac{\delta}{4(t+1)^2}}$ samples from $D^{(t)}$.  Let $g^{(t)} \leftarrow \calL_{\frac{\eps}{120}, \frac{\delta}{4(t+1)^2}, \calF}(S^{(t)})$.  \label{line:alg-main-1-2}
	\State Let $Z^{(t)} \leftarrow \utest(g^{(t)}, k, t, \eps, \delta)$.
	\For {\textbf{each} $i \in \{1, 2, \dots, k\}$}
		\If {$i \in Z^{(t)}$}
			\State $w_i^{(t+1)} \leftarrow  w_i^{(t)}$
		\Else
			\State $w_i^{(t+1)} \leftarrow 2 \cdot w_i^{(t)}$.
		\EndIf
	\EndFor
\EndFor
\State \textbf{return} $g = \mathrm{Plurality}(g^{(0)},  \dots, g^{(T-1)})$.
\end{algorithmic}
\end{algorithm}

\vspace{-3ex}
\begin{algorithm}[H]  
\caption{Accuracy Test ($\utest(g, k, t, \eps, \delta)$)}
\label{alg:test}
\begin{algorithmic}[1]  
\For {\textbf{each} $i \in \{1, 2, \dots, k\}$}
	Let $T_i$ be a set of $\frac{432}{\eps} \ln \left(\frac{k \cdot 4(t+1)^2}{\delta}\right)$ samples from $D_i$.  \label{line:alg-test-1}
\EndFor
\State \textbf{return} $\{i\ |\ \err_{T_i} (g) \leq \frac{\eps}{6}\}$.
\end{algorithmic}
\end{algorithm}

\end{multicols}

\begin{lemma}

\label{lem:utest}
With probability at least $1 - \frac{\delta}{ 4(t+1)^2}$, $\utest(g, k , t, \eps, \delta)$ returns a set of players that includes 1) each $i$ such that $\err_{D_i}(g) \leq \frac{\eps}{12}$, 2) none of the $i$ such that $\err_{D_i}(g) > \frac{\eps}{4}$. 
\end{lemma}

Given a function $g$ and a distribution $D$, we say that $g$ is a \emph{good candidate} for $D$ if $\err_D(g) \leq \frac{\eps}{4}$. The following lemma shows that if we have a set of functions where most of them are good candidates for $D$, then the plurality vote of these functions also has good accuracy for $D$.

\begin{lemma}\label{lem:maj}
Let $g_1, g_2, \dots, g_m$ be a set of functions such that more than $70\%$ of them are good candidates for $D$. Let $g = \mathrm{Plurality}(g_1, g_2, \dots, g_m)$, we have that $\err_D(g) \leq \eps$. 
\end{lemma}

We let the $\calE$ be the event that every call of the learner $\calL$ and $\utest$ is successful. It is straightforward to see that 
\begin{align}\label{eq:calEprob}
\Pr[\calE] \ge 1- \sum_{t=0}^{+\infty}  \frac{\delta}{4 (t+1)^2} \cdot 2 = 1 - \frac{\delta \cdot \pi^2}{24} > 1 - \delta .
\end{align}

Now we are ready to prove the main theorem for \pref{alg:main-1}.

\begin{theorem}\label{thm:main-1}
\pref{alg:main-1} has the following properties.
\begin{enumerate}
\item With probability at least $1 - \delta$, it returns a function $g$ such that $\err_{D_i}(g) \leq \eps$ for all $i \in \{1, 2, \dots, k\}$.
\item Its sample complexity is $\displaystyle{O\left(\frac{\ln k}{\eps}(d + k \ln \delta^{-1} + k \ln k) \right)}$.
\end{enumerate}
\end{theorem}
\begin{proof}
While the sample complexity is easy to verify, we focus on the proof of the first property. In particular, we show that when $\calE$ happens (which is with probability at least $1 - \delta$ by \eqref{eq:calEprob}), we have $\err_{D_i}(g) \leq \eps$ for all $i \in \{1, 2, \dots, k\}$.

For now till the end of the proof, we assume that $\calE$ happens.

For each round $t$, we have that 
$\frac{\eps}{120} \geq \err_{D^{(t)}}(g^{(t)}) = \E_{i \sim p^{(t)}(\cdot)} [\err_{D_i} (g^{(t)})]$ .
Therefore, by Markov inequality, we have that
$\Pr_{i \sim p^{(t)}(\cdot)} \left[\err_{D_i} (g^{(t)}) > \frac{\eps}{12}\right] \leq .1$ .
In other words, 
\begin{align}\label{eq:thm:main-1:1}
.1 \geq \sum_{i : \err_{D_i} (g^{(t)}) > \frac{\eps}{12}} p^{(t)}(i) = \frac{1}{\sum_{i=1}^{k} w_i^{(t)}} \sum_{i : \err_{D_i} (g^{(t)}) > \frac{\eps}{12}} w_i^{(t)} . 
\end{align}
Now consider the total weight $\sum_{i=1}^{k} w_i^{(t+1)}$, we have
\begin{align}\label{eq:thm:main-1:2}
\sum_{i=1}^{k} w_i^{(t+1)} = \sum_{i=1}^{k} w_i^{(t)} +  \sum_{i \not\in Z^{(t)} } w_i^{(t+1)} .
\end{align}
By \pref{lem:utest} and $\calE$, we have that 
\begin{align}\label{eq:thm:main-1:3}
 \sum_{i \not\in Z^{(t)} } w_i^{(t+1)} \leq \sum_{i : \err_{D_i}(g^{(t)}) > \frac{\eps}{12} } w_i^{(t+1)} .
\end{align}
Combining \eqref{eq:thm:main-1:1}, \eqref{eq:thm:main-1:2}, and \eqref{eq:thm:main-1:3}, we have
$\sum_{i=1}^{k} w_i^{(t+1)} \leq 1.1 \sum_{i=1}^{k} w_i^{(t)}$.
Since $\sum_{i=1}^{k} w_i^{(0)} = k$, we have the following inequality holds for every $t = 0, 1, 2, \dots$ :
$ \sum_{i=1}^{k} w_i^{(t)} \leq 1.1^t \cdot k$ .

Now let us focus on an arbitrary player $i$. We will show that for at least $70\%$ of the rounds $t$, we have $\err_{D_i} (g^{(t)}) \leq \frac{\eps}{4}$, and this will conclude the proof of this theorem thanks to \pref{lem:maj}. 

Suppose the contrary: for more than $30\%$ of the rounds, we have $\err_{D_i} (g^{(t)}) > \frac{\eps}{4}$. At each of such round $t$, we have $i \not\in Z^{(t)}$ because of \pref{lem:utest} and $\calE$, and therefore $w_i^{(t+1)} = 2\cdot w_i^{(t)}$. Therefore, we have  $w_i^{(T)} \geq 2^{.3 T}$. Together with \eqref{eq:thm:main-1:3}, we have
$2^{.3T} \leq w_i^{T} \leq \sum_{i=1}^{k} w_i^{(T)} \leq 1.1^T \cdot k$,
which is a contradiction for $T = 10 \ln k$.
\end{proof}

%% file: up-improved.tex
\section{The Quest for Optimality via Robust Multiplicative Weights}
\label{sec:improve}

In this section we improve the result in Theorem~\ref{thm:main-1} to get an optimal algorithm when $k$ is polynomially bounded by $d$ (see Theorem~\ref{thm:main-2}; the optimality will be shown in Section~\ref{sec:lb}).  In fact, our improved algorithm (\pref{alg:main-2} using \pref{alg:ntest} as a subroutine), is almost the same as \pref{alg:main-1} (using \pref{alg:test} as a subroutine). We highlight the differences as follows.

\begin{enumerate}
\item The total number of iterations at \pref{line:alg-main-1-1} of \pref{alg:main-1} is changed to $\tT = 2000 \ln (k/\delta)$.

\item The failure probability for the single-distribution learning algorithm $\calL$ at \pref{line:alg-main-1-2} of \pref{alg:main-1} is increased to a constant ${1}/{100}$.

\item The number of times that each distribution is sampled at \pref{line:alg-test-1} of \pref{alg:test} is reduced to $\frac{432}{\eps} \ln(100)$.
\end{enumerate}

Although these changes seem minor, it requires substantial technical efforts to establish \pref{thm:main-2}. We describe the challenge and sketch our solution as follows.

While the 2nd and 3rd items lead to the key reduction of the sample complexity, they make it impossible to use the union bound and claim that with high probability ``every call of $\calL$ and \utest~is successful'' (see Inequality \pref{eq:calEprob} in the analysis for \pref{alg:main-1}).

To address this problem, we will make our multiplicative weight analysis robust against occasionally failed rounds so that it works  when ``most calls of $\calL$ and \ntest~are successful''.

\begin{multicols}{2}

In more details, we will first work on the total weights $W^{(t)} = \sum_{i=1}^{k} w_i^{(t)}$ at the $t$-th round, and show that conditioned on the $t$-th round, $\E[W^{(t+1)}]$ is upper bounded by $1.13 W^{(t)}$ (where in contrast we had a stronger and deterministic statement $\sum_{i=1}^{k} w_i^{(t+1)} \leq 1.1 \sum_{i=1}^{k} w_i^{(t)}$ in the analysis for the basic algorithm). Using Jensen's inequality we will be able to derive that $\E[\ln W^{(t+1)}]$ is upper bounded by $(\ln 1.13 + \ln W^{(t)})$. Then, using Azuma's inequality for supermartingale random variables, we will show that with high probability, $\ln W^{(\tT)} \leq \tT (\ln 1.18) + \ln W^{(0)}$, i.e.\ $W^{(\tT)} \leq 1.18^{\tT} \cdot k$, which corresponds to $ \sum_{i=1}^{k} w_i^{(t)} \leq 1.1^t \cdot k$ in the basic proof. On the other hand, recall that in the basic proof we had to show that if for more than 30\% of the rounds, the $g^{(t)}$ function is not a good candidate for a player distribution $D_i$, then we have $w_i^{(T)} \geq 2^{.3 T}$. In the analysis for the improved algorithm, because the \ntest~procedure fails with much higher probability, we need to use concentration inequalities and derive a slightly weaker statement ($w_i^{(\tT)} \geq 2^{.25 \tT}$). Finally, we will put everything together using the same proof via contradiction argument, and prove the following theorem.

\begin{algorithm}[H]  
\caption{{\IMW}}
\label{alg:main-2}
\begin{algorithmic}[1]  
\State Let the initial weight $w_i^{(0)} \leftarrow 1$ for each player $i \in \{1, 2, 3, \dots, k\}$.
\State Let $\tT \leftarrow 2000 \ln (k/\delta)$. \label{line:alg-main-2-2}
\For {$t \leftarrow 0 \textbf{~to~} \tT - 1$}
	\State Let $p^{(t)}(i) \leftarrow \frac{w_i^{(t)}}{\sum_{i=1}^{k} w_i^{(t)}}$ for each $i \in \{1, 2, 3, \dots, k\}$ so that $p^{(t)}(\cdot)$ defines a probability distribution.
	\State Let $D^{(t)} \leftarrow \sum_{i = 1}^{K} p^{(t)}(i) D_i$.
	\State Let  $S^{(t)}$ be a set of $\bbS_{\frac{\eps}{120}, \frac{1}{100}}$ samples from $D^{(t)}$.  Let $g^{(t)} \leftarrow \calL_{\frac{\eps}{120}, \frac{1}{100}, \calF}(S^{(t)})$. \label{line:alg-main-2-6}
	\State Let $Z^{(t)} \leftarrow \ntest(k, g^{(t)}, \eps, \delta)$.
	\For {\textbf{each} $i \in \{1, 2, 3, \dots, k\}$}
		\If {$i \in Z^{(t)}$}
			\State $w_i^{(t+1)} \leftarrow  w_i^{(t)}$
		\Else
			\State $w_i^{(t+1)} \leftarrow 2 \cdot w_i^{(t)}$.
		\EndIf
	\EndFor
\EndFor
\State \textbf{return}  $g = \mathrm{Plurality} (g^{(0)}, \dots, g^{(\tT-1)})$.
\end{algorithmic}
\end{algorithm}
\vspace{-3ex}
\begin{algorithm}[H]  
\caption{Weak Accuracy Test ($\ntest(g, k, \eps, \delta)$)}
\label{alg:ntest}
\begin{algorithmic}[1]  
\For {\textbf{each} $i \in \{1, 2, 3, \dots, k\}$}
	Let $T_i$ be a set of $\frac{432}{\eps} \ln \left(100\right)$ samples from $D_i$.
\EndFor
\State \textbf{return} $\{i\ |\ \err_{T_i} (g) \leq \frac{\eps}{6}\}$.
\end{algorithmic}
\end{algorithm}
\end{multicols}

\begin{theorem}\label{thm:main-2}
\pref{alg:main-2} has the following properties.
\begin{enumerate}
\item With probability at least $1 - \delta$, it returns a function $g$ such that $\err_{D_i}(g) \leq \eps$ for all $i \in \{1, 2, \dots, k\}$.
\item Its sample complexity is $\displaystyle{O\left(\frac{(\ln k + \ln \delta^{-1})(d + k)}{\eps}\right)}$.
\end{enumerate}
\end{theorem}

Now we prove \pref{thm:main-2}.

Similarly to \pref{lem:utest}, applying \pref{prop:chernoff} (but without the union bound), we have the following lemma for \ntest.
\begin{lemma}
\label{lem:ntest}
For each player $i$, with probability at least $1- \frac{1}{100}$, the following hold, 1) if $\err_{D_i}(g) \leq \frac{\eps}{12}$, then $i \in \ntest(g, k , \eps, \delta)$; 2) if $\err_{D_i}(g) > \frac{\eps}{4}$, then $i \not\in \ntest(g, k ,  \eps, \delta)$.
\end{lemma}

Let the indicator variable $\psi_i^{(t)} = 1$ if the desired event described in \pref{lem:ntest} for $i$ and time $t$ does not happen; and let $\psi_i^{(t)} = 0$ otherwise. By \pref{lem:ntest}, we have $\E[\psi_i^{(t)}] \leq \frac{1}{100}$. By \pref{prop:chernoff}, for each player $i$, we have
$\Pr\left[\sum_{t=0}^{\tT-1} \psi_i^{(t)}  > .05 \tT\right] \leq \exp\left(- \frac{1}{3} \cdot 4^2 \cdot \frac{\tT}{100}\right) \leq \exp\left(- \frac{5 \tT}{100} \right) \leq \frac{\delta}{k^5}$.

Now let $\calJ_1$ be the event that $\sum_{t=0}^{\tT-1} \psi_i^{(t)}  \leq .05 \tT$ for every $i$. Via a union bound, we have that 
\begin{align} \label{eq:thm:main-2:1}
\Pr[\calJ_1] \geq 1 - \frac{\delta}{k^4} .
\end{align}

Let the indicator variable $\chi^{(t)} = 1$ if the learner $\calL$ fails at time $t$; and let $\chi^{(t)} = 0$ otherwise. We have 
\begin{align}\label{eq:thm:main-2:11}
\E\left[\chi^{(t)} \ | \ \text{time $0$, $1$, \dots, $t-1$} \right] \leq \frac{1}{100} .
\end{align} 


Let $W^{(t)} = \sum_{i=1}^{k}  w_i^{(t)}$ be the total weights at time $t$. For each $t$, similarly to \eqref{eq:thm:main-1:2}, we have
\begin{align} \label{eq:thm:main-2:3}
W^{(t+1)} = W^{(t)} + \sum_{i \not\in Z^{(t)}} w_i^{(t)} .
\end{align}
For each $i$ such that $\err_{D_i}(g^{(t)}) \leq \frac{\eps}{12}$, by \pref{lem:ntest}, we know that $\Pr[i \not\in Z^{(t)}] \leq \frac{1}{100}$. Therefore, if we take the expectation over the randomness of \ntest\ at time $t$, we have, 
\begin{eqnarray}
\E \left[\sum_{i \not\in Z^{(t)}} w_i^{(t)}\right] &\leq& \sum_{i : \err_{D_i}(g^{(t)}) > \frac{\eps}{12} } w_i^{(t)} + \E\left[\sum_{i : \err_{D_i}(g^{(t)}) \leq \frac{\eps}{12} } w_i^{(t)}\right] \nonumber \\ 
&\leq& \sum_{i : \err_{D_i}(g^{(t)}) > \frac{\eps}{12} } w_i^{(t)} + \frac{1}{100} \cdot \sum_{i=1}^{k} w_i^{(t)}.  \label{eq:thm:main-2:4}
\end{eqnarray}

When $\chi^{(t)} = 0$, similarly to the proof of \pref{thm:main-1}, we have  
$\Pr_{i \sim p^{(t)}(\cdot)} \left[\err_{D_i} (g^{(t)}) > \frac{\eps}{12}\right] \leq .1$,
and 
\begin{align}\label{eq:thm:main-2:5}
.1 \geq \sum_{i : \err_{D_i} (g^{(t)}) > \frac{\eps}{12}} p^{(t)}(i) = \frac{1}{\sum_{i=1}^{k} w_i^{(t)}} \sum_{i : \err_{D_i} (g^{(t)}) > \frac{\eps}{12}} w_i^{(t)}.
\end{align}
Combining \eqref{eq:thm:main-2:3}, \eqref{eq:thm:main-2:4}, and \eqref{eq:thm:main-2:5}, we have (when $\chi^{(t)} = 0$)
\begin{align}\label{eq:thm:main-2:10}
\E\left[W^{(t+1)}\ \big|\ \chi^{(t)} = 0 \text{~and~} W^{(0)}, \dots, W^{(t)}\right] \leq 1.11 \cdot W^{(t)} .
\end{align}
Together with \eqref{eq:thm:main-2:11}, we have
$\E\left[W^{(t+1)}\ \big|\ W^{(0)}, \dots, W^{(t)}\right] \leq 1.11 \cdot W^{(t)} 
= \ln \Big(\E\left[W^{(t+1)} \ \big|\ \chi^{(t)} = 0 \text{~and~} W^{(0)}, \dots, W^{(t)}\right]  \cdot \Pr\left[\chi^{(t)} = 0 \ | W^{(0)}, \dots, W^{(t)}\right]  
  + 2 W^{(t)} \cdot \Pr\left[\chi^{(t)} = 1 \ | W^{(0)}, \dots, W^{(t)}\right]\Big)
 \leq (1.11  + 0.02) W^{(t)} = 1.13 W^{(t)}$.

Let $Q^{(t)} = \ln W^{(t+1)} / W^{(t)} $, and by Jensen's inequality, we have 
$\E \left[Q^{(t)} \ \big| \  W^{(0)}, \dots, W^{(t)}\right] \leq \ln \E\left[W^{(t+1)}/ W^{(t)} \ \big| \ W^{(0)}, \dots, W^{(t)}\right]$.
Therefore, we have
$\E \left[Q^{(t)} \ \big| \  Q^{(0)}, \dots, Q^{(t-1)}\right] \nonumber = \E \left[Q^{(t)} \ \big| \  W^{(0)}, \dots, W^{(t)}\right]  
 \leq  \ln \E\left[W^{(t+1)}/ W^{(t)} \ \big| \ W^{(0)}, \dots, W^{(t)}\right] 
 \leq  \ln (1.11 + .02) = \ln 1.13$.

Now let $\tQ^{(t)} = \sum_{z=0}^{t - 1} Q^{(z)} - t \cdot \ln 1.13$ for all $t =0, 1, 2, \dots$. We have that $\{\tQ^{(t)}\}$ is a supermartingale and $|\tQ^{(t+1)} - \tQ^{(t)}| \leq \ln 2$ for all $t =0, 1, 2, \dots$. By \pref{prop:azuma} and noticing that $\ln 1.18 - \ln 1.13 > .04 $, we have
$\Pr\left[\sum_{t=0}^{\tT - 1} Q^{(t)} > (\ln 1.18) \tT\right] \leq  \Pr\left[\tQ^{(\tT)} - \tQ^{(0)} > .04 \tT\right] \leq \exp\left( - \frac{.04^2 \cdot \tT}{2 \cdot (\ln 2)^2}\right) \leq \frac{\delta}{k^2}$.
Let $\calJ_2$ be the event that $W^{(\tT)} \leq 1.18^{\tT} \cdot k \Leftrightarrow \sum_{t=0}^{\tT - 1} Q^{(t)} \leq (\ln 1.18) \tT$, we have that
\begin{align}\label{eq:thm:main-2:6}
\Pr[\calJ_2] \geq 1 - \frac{\delta}{k^2}.
\end{align}

Now let $\calJ = \calJ_1 \cap \calJ_2$, combining \eqref{eq:thm:main-2:1} and \eqref{eq:thm:main-2:6}, for $k \geq 2$, we have
\begin{align}\label{eq:thm:main-2:7}
\Pr[\calJ] \geq 1 - \frac{\delta}{k}.
\end{align}

Now we are ready to prove \pref{thm:main-2} for \pref{alg:main-2}.

\begin{proof}[of \pref{thm:main-2}]
While the sample complexity is easy to verify, we focus on the proof of the first property. In particular, we show that when $\calJ$ happens (which is with probability at least $1 - \delta$ by \eqref{eq:thm:main-2:7}), we have $\err_{D_i}(g) \leq \eps$ for all $i \in \{1, 2, 3, \dots, k\}$.

Let us consider an arbitrary player $i$. We will show that when $\calJ$ happens, for at least $70\%$ the times $t$, we have $\err_{D_i} (g^{(t)}) \leq \frac{\eps}{4}$, and this will conclude the proof of this theorem thanks to \pref{lem:maj}. 

Suppose the contrary: for more than $30\%$ of the times, we have  $\err_{D_i} (g^{(t)}) > \frac{\eps}{4}$. Because of $\calJ_1$, for more than $30\% - 5\% = 25\%$ of the times $t$, we have $i \not\in Z^{(t)}$. Therefore, we have
$w_i^{(\tT)} \geq 2^{.25 \tT}$.
On the other hand, by $\calJ_2$ we have $W^{(\tT)} \leq 1.2^{\tT}$. Therefore, we reach
$2^{.25 \tT} \leq w_i^{(\tT)} \leq W^{(\tT)} \leq 1.18^{\tT} \cdot k$,
which is a contradiction to $\tT = 2000 \ln(k/\delta)$.
\end{proof}

%% file: lb.tex
\section{Lower Bound}
\label{sec:lb}

We show the following lower bound result, which matches our upper bound (\pref{thm:main-1}) when $k = (1/\delta)^{\Omega(1)}$ and $k = d^{O(1)}$. 

\begin{theorem}
\label{thm:lb}

In collaborative PAC learning with $k$ players and a hypothesis class of VC-dimension $d$, for any $\eps, \delta \in (0, 0.01)$, there exists a hard input distribution on which any $(\eps, \delta)$-learning algorithm $\calA$ needs $\Omega(\max\{d \ln k, k \ln d\}/\eps)$ samples in expectation, where the expectation is taken over the randomness used in obtaining the samples and the randomness used in drawing the input from the input distribution.
\end{theorem}

The proof of \pref{thm:lb} is similar to that for the lower bound result in \cite{BHPQ17}; however, we need to generalize the hard instance provided in \cite{BHPQ17} in two different cases.  We briefly discuss the high level ideas of our generalization here, and leave the full proof to \pref{app:lb} due to space constraints.

The lower bound proof in \cite{BHPQ17} (for $k = d$) performs a reduction from a simple player problem to a $k$-player problem, such that if we can $(\eps, \delta)$-PAC learn the $k$-party problem using $m$ samples in total, then we can $(\eps, 10\delta/(9k))$-PAC learn the single player problem using $O(m/k)$ samples. Now for the case when $d > k$, we need to change the single player problem used in \cite{BHPQ17} whose hypothesis class is of VC-dimension $\Theta(1)$ to one whose hypothesis class is of VC-dimension $\Theta(d/k)$.  For the case when $d \le k$, we essentially duplicate the hard instance for a $d$-player problem $k/d$ times, getting a hard instance for a $k$-player problem, and then perform the random embedding reduction from the single player problem to the $k$-player problem. See \pref{app:lb} for details.

%% file: exp.tex
\newcommand{\chensays}[2][]{\textcolor{red} {\textsc{Chen #1:} \emph{#2}}}
\newcommand{\letter}{{\sc Letter}}
\newcommand{\magiceven}{{\sc Magic-Even}}
\newcommand{\magicone}{{\sc Magic-1}}
\newcommand{\magictwo}{{\sc Magic-2}}
\newcommand{\wine}{{\sc Wine}}
\newcommand{\winetwo}{{\sc Wine-2}}
\newcommand{\eye}{{\sc Eye}}

\newcommand{\naive}{{\sc Naive}}
\newcommand{\blum}{{\sc CenLearn}}
\newcommand{\basic}{{\sc Basic}}
\newcommand{\improved}{{\sc MWeights}}

\newcommand{\avg}{\text{avg}}

\section{Experiments}
\label{sec:exp}
We present in this section a set of experimental results which demonstrate the effectiveness of our proposed algorithms.

Our algorithms are based on the assumption that given a hypothesis class, we are able to compute its VC dimension $d$ and access an oracle to compute an $(\eps, \delta)$-classifier with sample complexity $\bbS_{\eps, \delta}$. In practice, however, it is usually computationally difficult to compute the exact VC dimension for a given hypothesis class. Also, the VC dimension usually only proves to be a very loose upper bound for the sample complexity needed for an $(\eps, \delta)$-classifier.

To address these practical difficulties, in our experiment, we treat the VC dimension $d$ as a parameter to control the sample budget. More specifically, we will first choose a concrete model as the oracle; in our implementation, we choose the decision tree. We then set the parameter $\delta = 0.9$ and gradually increase $d$ to determine the sample budget.
For each fixed sample budget (i.e., each fixed $d$), we run the algorithm for $100$ times and test whether the following happens,
\begin{equation}
  \label{eq:goal}
  \widehat{\Pr}[\max_i\err_{D_i}(g) \leq \eps~\text{for all}~ i] \geq 0.9.
\end{equation}
Here $\eps$ is a parameter we choose and $g$ is the classifier returned by the collaborative learning algorithm to be tested. The empirical probability $\widehat{\Pr}[\cdot]$ in \eqref{eq:goal} is calculated over the $100$ runs. We finally report the minimum number of samples consumed by the algorithm to achieve \eqref{eq:goal}.

Note that in our theoretical analysis, we did not try to optimize the constants. Instead, we tune the constants for both \blum\ and \improved\ for better performance. Please find more implementation details in the appendix.

\paragraph{Datasets.}  We will test the collaborative learning algorithms using the following data sets.
\begin{itemize}

\item \magiceven~\cite{BAM04}. This data set is generated to simulate registration of high energy gamma particles in an atmospheric Cherenkov telescope. There are $19,020$ instances and each belongs to one of the two classes (gamma and hadron). There are $11$ attributes in each data point. We randomly partition this data set into $k=10$ subsets (namely, $D_1, \ldots, D_k$).

\item \magicone. The raw data set is the same as we have in \magiceven. Instead of random partitioning, we partition the data set into $D_1$ and $D_2$ based on the two different classes, and make $k-2$ more copies of $D_2$ so that $D_2, D_3, \ldots, D_k$ are identical. In our case we set $k = 10$.

\item \magictwo. This data set differs from \magicone\ in the way of constructing $D_1$ and $D_2$: we partition the original data set into $D_1$ and $D_2$ based on the first dimension of the feature vectors; we then make duplicates for $D_2$. Here we again set $k = 10$.

\item \wine~\cite{PAF09}. This data set contains physicochemical tests for white wine, and the scores of the wine range from $0$ to $10$. There are $4,898$ instances and there are $12$ attributes in the feature vectors. We partition the data set into $D_1, \ldots, D_4$ based on the first two dimensions.

\item \eye. This data set consists of 14 EEG values and a value indicating the eye state. There are $14,980$ instances in this data set. We partition it into $D_1, \ldots, D_4$ based on the first two dimensions. 

\item \letter~\cite{WS91}.
This data set has $20,000$ instances, each in $\mathbb{R}^{16}$. There are $26$ classes, each representing one of $26$ capital letters. We partition this data set into $k=12$ subsets based on the first $4$ dimensions of the feature vectors.
  
\end{itemize}

\paragraph{Tested Algorithms.}
We compare our algorithms with the following two baseline algorithms,
\begin{itemize}
\item \naive. In this algorithm we treat all distributions $D_1, \ldots, D_k$ equally. That is, given a budget $z$, we sample $z$ training samples from $D = \frac{1}{k}\sum_{i=1}^kD_i$. We then train a classifier (decision tree) using those samples.
\item \blum, this is the implementation of the algorithm proposed by Blum et al.~\cite{BHPQ17}.
\end{itemize}

Since our \pref{alg:main-1} and \pref{alg:main-2} are very similar, and \pref{alg:main-2} has better theoretical guarantee, we will only test \pref{alg:main-2}, denoted as \improved, in our experiments.

\paragraph{Experimental Results and Discussion.}

\begin{figure}[t]
     \centering
     \subfloat[][\magiceven]{\includegraphics[width=5cm,height=4cm]{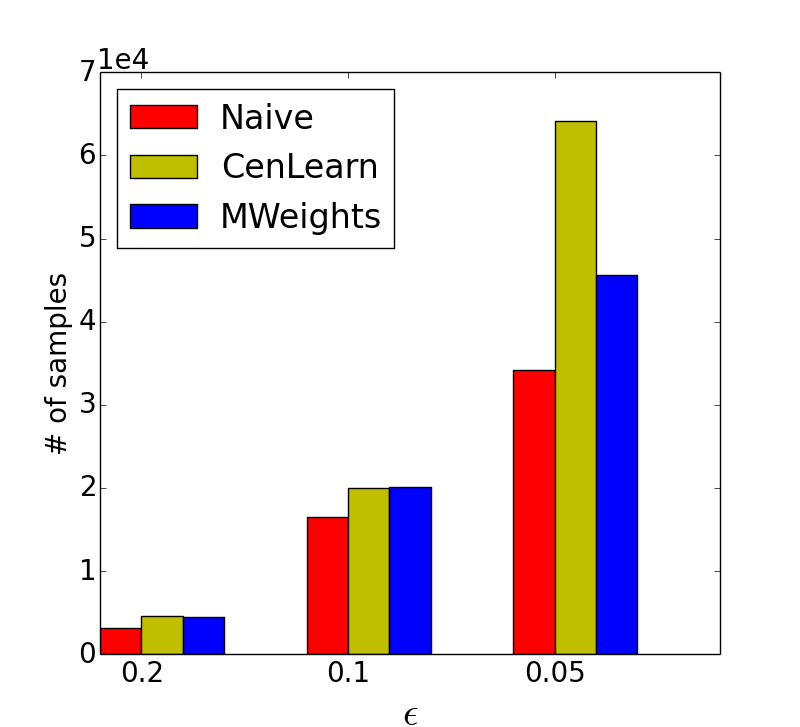}\label{fig:magic-even}}
     \hspace*{-1.5em}
     \subfloat[][\magicone]{\includegraphics[width=5cm,height=4cm]{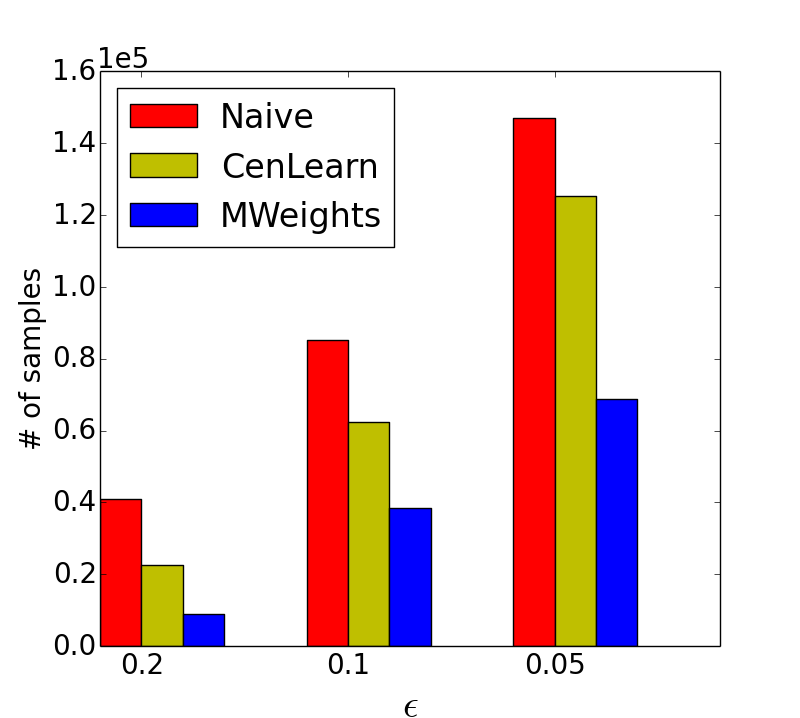}\label{fig:magic-1}}
     \hspace*{-1.5em}
     \subfloat[][\magictwo]{\includegraphics[width=5cm,height=4cm]{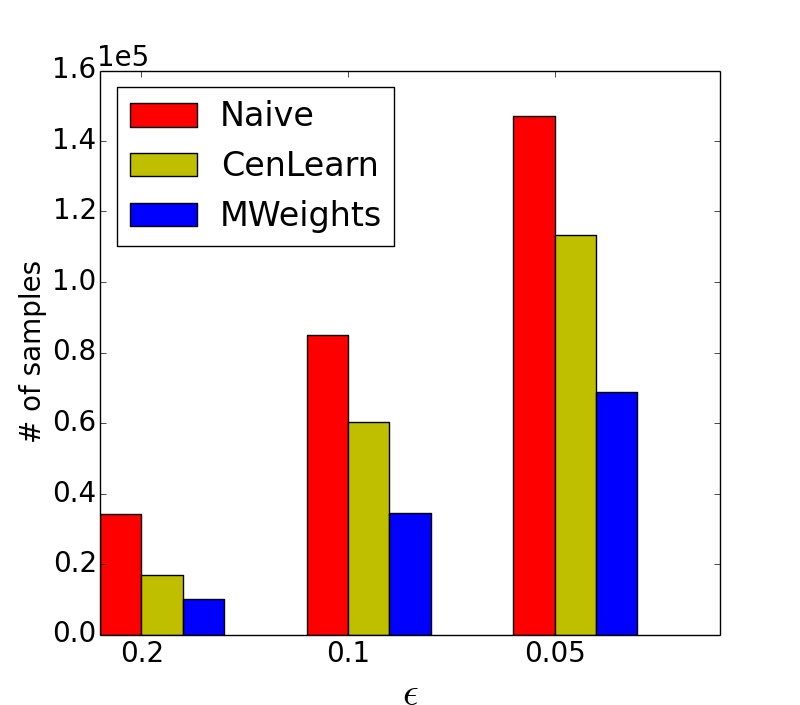}\label{fig:magic-2}}
     \\
     \subfloat[][\wine]{\includegraphics[width=5cm,height=4cm]{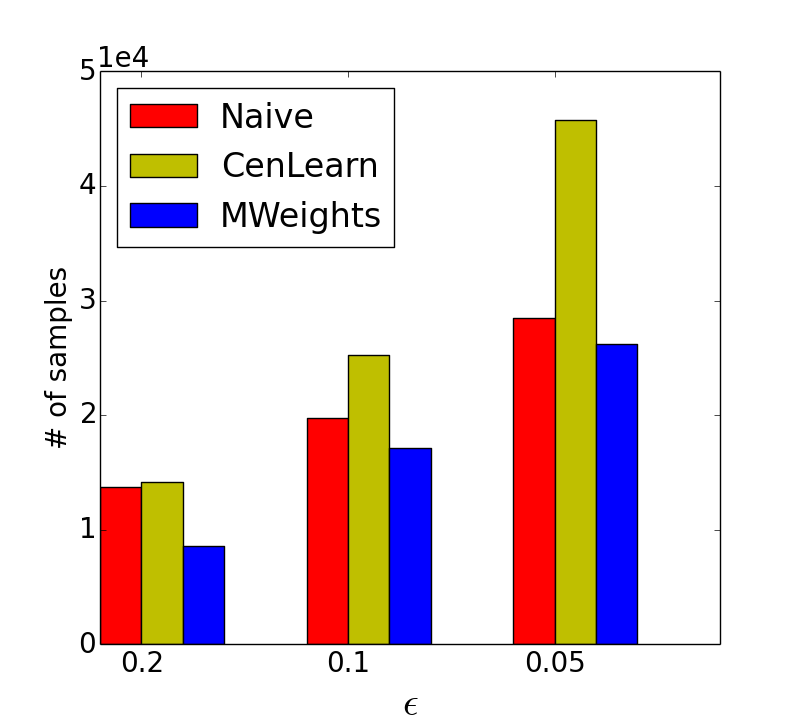}\label{fig:wine}}
     \hspace*{-1.5em}
     \subfloat[][\eye]{\includegraphics[width=5cm,height=4cm]{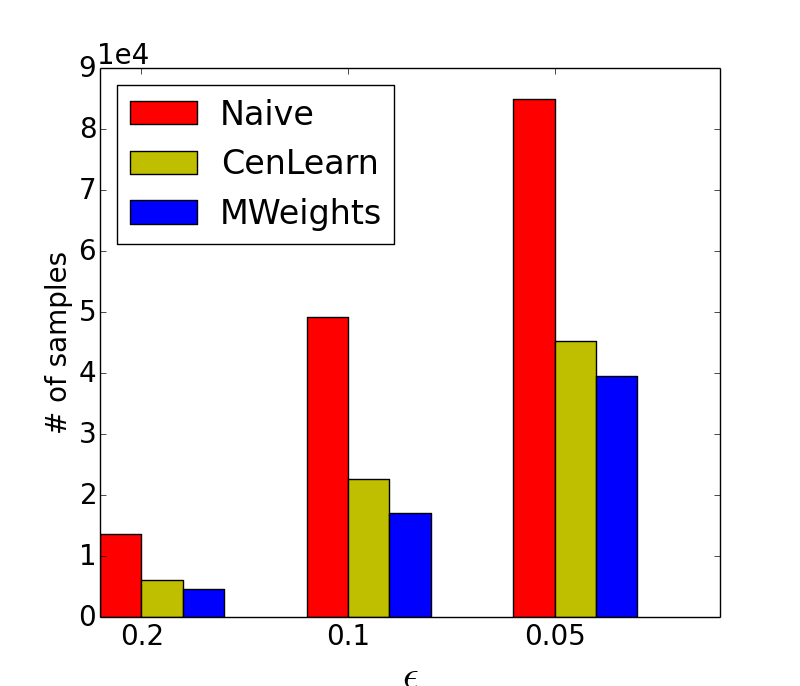}\label{fig:eye}}
     \hspace*{-1.5em}
     \subfloat[][\letter]{\includegraphics[width=5cm,height=4cm]{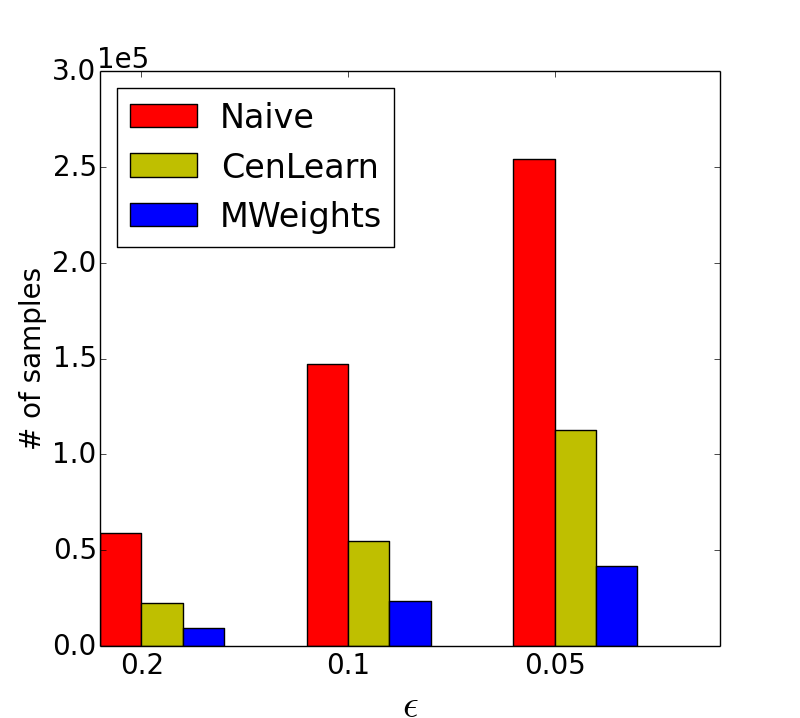}\label{fig:letter}}
     \caption{Sample complexity versus error threshold $\eps$.}
     \label{fig:dt}
\end{figure}

The experimental results are presented in \pref{fig:dt}. We test the algorithms for each data set using multiple values of the error threshold $\eps$, and report the sample complexity for \naive, \improved\ and \blum.

In \pref{fig:magic-even}, we notice that \naive\ uses less samples than its competitors. This phenomenon is predictable because in \magiceven, $D_1, \ldots, D_k$ are constructed via random partitioning, which is the easiest case for \naive. Since \improved\ and \blum\ need to train multiple classifiers, each classifier will get fewer training samples than \naive\ when the total budgets are the same.

In \pref{fig:magic-1} and \pref{fig:magic-2}, $D_1, \ldots, D_k$ are constructed in a way  that $D_2, D_3, \ldots, D_k$ are identical, and $D_1$ is very different from other distributions. Thus the overall distribution (i.e., $D = \frac{1}{k}\sum_{i=1}^kD_i$) used to train \naive\ is quite different from the original data set. One can observe from those two figures that \improved\ still works quite well while \naive\ suffers.
 
In \pref{fig:magic-1}-\pref{fig:letter}, one can observe that \improved\ uses fewer samples than its competitors in almost all cases, which shows the superiority of our proposed algorithm. 
\blum\ outperforms \naive\ in general.  However, \naive\ uses slightly fewer samples than \blum\ in some cases (e.g., \pref{fig:wine}). This may due to the fact that the distributions $D_1, \ldots, D_k$ in those cases are not hard enough to show the superiority of \blum\ over \naive.

To summarize, our experimental results show that \improved\ and \blum\ need fewer samples than \naive\ when the input distributions $D_1, \ldots, D_k$ are sufficiently different. \improved\ consistently outperforms \blum, which may due to the facts that \improved\ has better theoretical guarantees and is more straightforward to implement.

%% file: conclusion.tex
\section{Conclusion}
\label{sec:conclusion}

In this paper we consider the collaborative PAC learning problem. We have proved the optimal overhead ratio and sample complexity, and conducted experimental studies to show the superior performance of our proposed algorithms. 

One open question is to consider the \emph{balance} of the numbers of queries made to each player, which can be measured by the ratio between the largest number of queries made to a player and the average number of queries made to the $k$ players. The proposed algorithms in this paper may attain a balance ratio of $\Omega(k)$ in the worst case.  It will be interesting to investigate:
\begin{enumerate}
\item Whether there is an algorithm with the same sample complexity but better balance ratio? 
\item What is the optimal trade-off between sample complexity and balance ratio?
\end{enumerate}

%% file: appendix.tex
\section{Concentration Bounds}

\begin{proposition}[Multiplicative Chernoff bound]\label{prop:chernoff}
Let $X_i (1 \leq i \leq n)$ be independent random variables with values in $[0,1]$.
Let $X =  \frac{1}{n} \sum_{i=1}^{n} X_i$. For every $0 \leq \epsilon \leq 1 $, we have that
\[
 \Pr\big[X < (1-\epsilon) \E[X] \big] < \exp\left( - \frac{\epsilon^2 n \E[X]}{2} \right),
\]
\[
 \Pr\big[ X > (1+\epsilon) \E[X] \big] < \exp\left(-\frac{\epsilon^2 n \E[X]}{3}\right).
\]
\end{proposition}

\begin{definition}[Supermartingale Random Variables]
A discrete-time \emph{supermartingale} is a sequence of random variables $X_0, X_1, X_2, \dots$ that satisfies for any time $t$,
\[
\E |X_t| < \infty, \text{~and~} \E [X_{t + 1} | X_0, \dots, X_t] \leq X_t .
\]
\end{definition}

\begin{proposition}[Azuma's inequality for supermartingale random variables]\label{prop:azuma}
Suppose $\{X_k : k = 0, 1, 2, \dots\}$ is a supermartingale and $|X_k - X_{k+1}| \leq c_k$ almost surely. Then for all positive integers $T$ and all positive reals $\theta$,
\[
\Pr[X_T - X_0 \geq \theta] \leq \exp\left ({-\theta^2 \over 2 \sum_{k=0}^{T-1} c_k^2} \right). 
\]
\end{proposition}

\section{Omitted Proofs in \pref{sec:basic}}\label{app:basic}

\begin{proof}[of \pref{lem:utest}]
For each $i$ such that $\err_{D_i}(g) \leq \frac{\eps}{12}$, by \pref{prop:chernoff}, we have that 
\[
\Pr\left[\err_{T_i} (g) > \frac{\eps}{2}\right] \leq \exp\left(-  \frac{432}{\eps}\ln \left(\frac{k \cdot 4(t+1)^2}{\delta}\right) \cdot \frac{\eps}{12^2 \cdot 3} \right) = \frac{\delta}{k\cdot 4(t+1)^2}.
\]
Therefore, with probability at least $1 - \frac{\delta}{k \cdot 4(t+1)^2}$, $i$ is included in the output of \utest.

Similarly, for each $i$ such that $\err_{D_i}(g) > \frac{\eps}{4}$, by \pref{prop:chernoff}, we have that 
\[
\Pr\left[\err_{T_i} (g) \leq \frac{\eps}{2}\right] \leq \exp\left(-  \frac{432}{\eps}\ln \left(\frac{k \cdot 4(t+1)^2}{\delta}\right) \cdot \frac{\eps}{12^2 \cdot 2} \right) \leq \frac{\delta}{k\cdot 4(t+1)^2}.
\]
Therefore, with probability at least $1 - \frac{\delta}{k \cdot 4(t+1)^2}$, $i$ is not included in the output of \utest.

The lemma is now proved by a union bound over at most $k$ players.
\end{proof}

\begin{proof}[of \pref{lem:maj}]
Suppose for contradiction that $\err_D(g) > \eps$. Given a sample $(x, y) \sim D$, when $g(x) \neq y$, we know that for more than half of the $g_i$'s, we have $g_i(x) \neq y$. Therefore, we have 
\begin{align}\label{eq:lem:maj:1}
\sum_{i = 1}^{m} \Pr_{(x, y) \sim D}[g(x) \neq y \text{~and~} g_i(x) \neq y] > \frac{\eps m}{2} .
\end{align}
On the other hand, by discussing whether $g_i$ is a good candidate for $D$, we have
\begin{eqnarray*}
&& \sum_{i = 1}^{m} \Pr_{(x, y) \sim D}[g(x) \neq y \text{~and~} g_i(x) \neq y] 
\\ 
&\leq&  \sum_{i : g_i \text{~good}}\Pr_{(x, y) \sim D}[g_i(x) \neq y]  + \sum_{i: g_i \text{~not good}} \Pr_{(x, y) \sim D}[g(x) \neq y]  \\
&\leq& \sum_{i : g_i \text{~good}} \frac{\eps}{4} + \sum_{i : g_i \text{~not good}} \eps \leq .7 m \cdot \frac{\eps}{4} + .3 m \cdot \eps < .5 m \eps ,
\end{eqnarray*}
which contradicts \eqref{eq:lem:maj:1}.
\end{proof}

\section{Proof of Theorem~\ref{thm:lb}}
\label{app:lb}

Before proving Theorem~\ref{thm:lb} we need a result from \cite{EHKV89}. Let $\phi_d$ be the following input distribution. 
\begin{itemize}
\item Instance space $\calY_{d} = \{0, 1, \ldots, d-1, \perp\}$.

\item Hypothesis class: $\calG_{d}$ is the collection of all binary functions on $\calY_{d}$ that map $\perp$ to $0$.

\item Target function: $g^*$ is chosen uniformly at random from $\calG_{d}$.

\item Player's distribution: $\Pr[\perp] = 1 - 8\eps$, and $\Pr[0] = \ldots \Pr[d-1] = 8\eps/k$.
\end{itemize}

\begin{lemma}[\cite{EHKV89}]
\label{lem:primitive}
For any $\eps, \delta \in (0, 0.01)$, any $(\eps, \delta)$-learning algorithm $\calA$ on $\phi_{d}$ needs $\Omega(d/\eps)$ samples in expectation, where the expectation is taken over the randomness used in obtaining the samples and the randomness used in drawing the input from $\phi_{d}$.
\end{lemma}

We prove Theorem~\ref{thm:lb} in two cases: $d > k$ and $d \le k$.

\paragraph{The case $d > k$.}
Let $\sigma(i, j) = (i - 1) \cdot d/k + j$. We create the following hard input distribution, denoted by $\Phi_{k,d}$.  

\begin{itemize}
\item Instance space: $\calX_d = \{0, 1, \ldots, d-1, \perp\}$.

\item Hypothesis class: $\calF_d$ is the collection of all binary functions on $\calX_d$ that map $\perp$ to $0$.

\item Target function: $f^*$ is chosen uniformly at random from $\calF_d$.

\item Player $i$'s distribution $D_i$ (for each $i \in [k]$): Assigns weights to items in $\{\sigma(i,0), \sigma(i,1), \ldots, \sigma(i,d-1), \perp\}$ as follows: $\Pr[\perp] = 1 - 8\eps$, and $\Pr[\sigma(i,0)] = \ldots \Pr[\sigma(i,d/k-1)] = 8\eps/k$.  For any other item $x \in \calX_d$, $\Pr[x] = 0$.
\end{itemize}
Note that the induced input distribution for the $i$-th player is the same as $\phi_{d/k}$ for any $i \in [k]$.

We have the following lemma.   
It is easy to see that Lemma~\ref{lem:reduction} and Lemma~\ref{lem:primitive} imply a sample complexity $\Omega(d \ln k / \eps)$ for any $(\eps, \delta)$-learning algorithm on input distribution $\Phi_{k,d}$ in expectation.

\begin{lemma}
\label{lem:reduction}

If there exists an $(\eps, \delta)$-learning algorithm $\calA'$ that uses $m$ samples in expectation on input distribution $\Phi_{k,d}$, then there exists an $(\eps, \frac{10}{9k} \cdot \delta)$-learning algorithm $\calA$ that uses $\frac{10}{9k} \cdot m$ samples in expectation on input distribution $\phi_{d/k}$.
\end{lemma}

\begin{proof}
We construct $\calA'$ for input distribution $\phi_{d/k}$ using $\calA$ for input distribution $\Phi_{k,d}$ as follows.
\begin{enumerate}
\item $\calA'$ draws an input instance $(\calF_d, f^*, \{D_i\}_{i \in [k]})$ from $\Phi_{k,d}$, and samples $\ell$ uniformly at random from $[k]$.

\item $\calA'$ simulates $\calA$ on instance $(\calF_d, f^*, \{D_i\}_{i \in [k]})$ with the input distribution of the $\ell$-th player replaced by $\phi_{d/k}$.  Every time $\calA$ draws a sample from player $i \neq \ell$, $\calA'$ does the same (which is free since $\calA'$ already knows $(\calF_d, f^*, \{D_i\}_{i \in [k]})$), and passes the sample (and its label) to $\calA$. Every time $\calA$ draws a sample from player $\ell$, $\calA'$ samples from distribution $\phi_{d/k}$ instead.  Let $(u, v)\ (u \in \{0, 1, \ldots, d/k-1, \perp\}, v \in \{0,1\})$ be the sample.  If $u = \perp$ then $\calA'$ passes $(\perp, 0)$ to $\calA$, otherwise $\calA'$ passes $(\sigma(\ell, u), v)$ to $\calA$.

\item When $\calA$ terminates and returns a function $f$ on $\calX_d$, $\calA'$ checks whether the error of $f$ on each $D_i\ (i \neq \ell)$ is no more than $\eps$. If yes, $\calA'$ returns $f'$ defined as $f'(\perp) = f(\perp)$, and $f'(u) = f(\sigma(\ell, u))$.  Otherwise $\calA'$ repeats the simulation on a new input instance from $\Phi_{k,d}$.
\end{enumerate}

We have the following claims, whose proofs can be found in \cite{BHPQ17} for a similar reduction.  The two claims finish the proof of Lemma~\ref{lem:reduction}.
\begin{claim}
\label{cla:error-reduction}
$\calA'$ is an $(\eps, \frac{10}{9k} \cdot \delta)$-learning algorithm on $\phi_{d/k}$, where $\delta$ is failure probability of $\calA$.
\end{claim}

\begin{claim}
\label{cla:comm-reduction}
$\calA'$ uses at most $\frac{10}{9k} \cdot m$ samples in expectation, where $m$ is the sample complexity of $\calA$.
\end{claim}

\paragraph{The case $d \le k$.} We againt start by constructing a hard input distribution for the $k$ players, denoted by $\Psi_{k,d}$.  We first construct a hard input distribution for the first $d$ players.  The construction is the same as the one used in \cite{BHPQ17} for the case $k = d$.

\begin{itemize}
\item Instance space: $\calX_d = \{1, 2, \ldots, d, \perp\}$.

\item Hypothesis class: $\calF_d$ is the collection of all binary functions on $\calX_d$ that map $\perp$ to $0$.

\item Target function: $f^*$ is chosen uniformly at random from $\calF_d$.

\item Player's distribution $D_i\ (i \in [d])$: with probability $1/2$, the $i$-th player assigns weights to items in $\{1, 2, \ldots, \perp\}$ as $\Pr[\perp] = 1$ and $\Pr[x] = 0$ for all other items $x \in \calX_d$; with probability $1/2$, it assigns weights as $\Pr[\perp] = 1 - 2\eps$, $\Pr[i] = 2\eps$, and $\Pr[x] = 0$ for all other items $x \in \calX_d$.
\end{itemize}
We then assign the same input distribution for the next $d$ players, the next next $d$ players, and so on.  In other words, we duplicate the input distribution of the first $d$ players for $k/d$ times.  Finally we randomly permute the $k$ players. 


Let $\psi$ denote the input distribution of $\Psi_{1,1}$.  We have the following lemma.
\begin{lemma}[\cite{BHPQ17}]
\label{lem:primitive-2}
For any $\eps, \delta \in (0, 0.01)$, any $(\eps, \delta)$-learning algorithm $\calA$ on $\psi$ needs $\Omega(\log(1/\delta)/\eps)$ samples in expectation, where the expectation is taken over the randomness used in obtaining the samples and the randomness used in drawing the input from $\psi$.
\end{lemma}

We use the following reduction.
\begin{enumerate}
\item $\calA'$ draws an input instance $(\calF_d, f^*, \{D_i\}_{i \in [k]})$ from $\Psi_{k,d}$, and samples $\ell$ uniformly at random from $[k]$.

\item $\calA'$ simulates $\calA$ on instance $(\calF_d, f^*, \{D_i\}_{i \in [k]})$ with the input distribution of the $\ell$-th player replaced by $\psi$.  Every time $\calA$ draws a sample from player $i \neq \ell$, $\calA'$ does the same (which is free since $\calA'$ already knows $(\calF_d, f^*, \{D_i\}_{i \in [k]})$), and passes the sample (and its label) to $\calA$. Every time $\calA$ draws a sample from player $\ell$, $\calA'$ samples from distribution $\psi$ instead.  Let $(u, v)\ (u \in \{1, \perp\}, v \in \{0,1\})$ be the sample. If $u = \perp$ then $\calA'$ passes $(\perp, 0)$ to $\calA$, otherwise $\calA'$ passes $(\ell, v)$ to $\calA$.

\item When $\calA$ terminates and returns a function $f$ on $\calX_d$, $\calA'$ checks whether the error of $f$ on each $D_i\ (i \neq \ell)$ is no more than $\eps$. If yes, $\calA'$ returns $f'$ defined as $f'(\perp) = f(\perp)$, and $f'(1) = f(\ell)$.  Otherwise $\calA'$ repeats the simulation on a new input instance from $\Psi_{k,d}$.
\end{enumerate}

Claim~\ref{cla:comm-reduction} still holds for the above reduction.  While Claim~\ref{cla:error-reduction} changes slightly to the following (by replacing $k$ in Claim~\ref{cla:error-reduction} to $d$).  
\begin{claim}
\label{cla:error-reduction-2}
$\calA'$ is an $(\eps, \frac{10}{9d} \cdot \delta)$-learning algorithm for the primitive problem, where $\delta$ is the failure probability of $\calA$.
\end{claim}

The proof is very similar to that for Claim~\ref{cla:error-reduction}.  The only difference is the following:  Let $p_i$ be the probability that on a random input instance sampled from $\Psi_{k,d}$, the function $f$ returned by $\calA$ satisfies $\err_{D_\ell}(f)> \eps$ and $\err_{D_i}(f) \le \eps$ for any $i \neq \ell$.  We now have $\sum_{i \in [k]} p_i \le k/d \cdot \delta$ (due to the $k/d$ times of duplication of the input distribution for the first $d$ players), instead of $\sum_{i \in [k]} p_i \le \delta$ as the case for Claim~\ref{cla:error-reduction}.  This difference makes the final failure bound to be $ \frac{10}{9d} \cdot \delta$ instead of $\frac{10}{9k} \cdot \delta$ compared with Claim~\ref{cla:error-reduction}. 

\medskip
The $\Omega(k \ln d \log(1/\delta) / \eps)$ lower bound follows from Lemma~\ref{lem:primitive-2}, Lemma~\ref{cla:comm-reduction} and Lemma~\ref{cla:error-reduction-2}.
\end{proof}

\section{Experiment Implementation Details}
As mentioned, we did not try to optimize constants in our theoretical analysis. In our experiment, we tuned several parameters for both \blum\ and \improved~for better empirical performance. In particular, we made the following changes.
\begin{itemize}
\item We set $\bbS_{\eps, \delta} = \frac{d + \log \delta^{-1}}{10\eps}$.
\item In both \improved\ and \blum, we set the number of iterations ($\tT$ in \improved~ and $t$ in \blum) to $\lceil 10 \log k \rceil$.
\item In \ntest~(\pref{alg:ntest}) of \improved\ and the TEST process in \blum, we only drew $30/\eps$ samples from $D_i$ and returned $\{i\ |\ \err_{T_i} (g) \leq \frac{\eps}{2}\}$.
\end{itemize}

%% file: nips_2018.bbl
\begin{thebibliography}{10}

\bibitem{BBFM12}
M.~Balcan, A.~Blum, S.~Fine, and Y.~Mansour.
\newblock Distributed learning, communication complexity and privacy.
\newblock In {\em COLT}, pages 26.1--26.22, 2012.

\bibitem{BEL13}
M.~Balcan, S.~Ehrlich, and Y.~Liang.
\newblock {Distributed $k$-means and $k$-median clustering on general
  communication topologies}.
\newblock In {\em NIPS}, pages 1995--2003, 2013.

\bibitem{BHPQ17}
A.~Blum, N.~Haghtalab, A.~D. Procaccia, and M.~Qiao.
\newblock Collaborative {PAC} learning.
\newblock In {\em NIPS}, pages 2389--2398, 2017.

\bibitem{BAM04}
R.~Bock, A.~Chilingarian, M.~Gaug, F.~Hakl, T.~Hengstebeck, M.~Jirina,
  J.~Klaschka, E.~Kotrc, P.~Savick{\`y}, S.~Towers, et~al.
\newblock Methods for multidimensional event classification: a case study.
\newblock {\em as Internal Note in CERN}, 2003.

\bibitem{PAF09}
P.~Cortez, A.~Cerdeira, F.~Almeida, T.~Matos, and J.~Reis.
\newblock Modeling wine preferences by data mining from physicochemical
  properties.
\newblock {\em Decision Support Systems}, 47(4):547--553, 2009.

\bibitem{EHKV89}
A.~Ehrenfeucht, D.~Haussler, M.~J. Kearns, and L.~G. Valiant.
\newblock A general lower bound on the number of examples needed for learning.
\newblock {\em Inf. Comput.}, 82(3):247--261, 1989.

\bibitem{freund1997decision}
Y.~Freund and R.~E. Schapire.
\newblock A decision-theoretic generalization of on-line learning and an
  application to boosting.
\newblock {\em Journal of computer and system sciences}, 55(1):119--139, 1997.

\bibitem{WS91}
P.~W. Frey and D.~J. Slate.
\newblock Letter recognition using holland-style adaptive classifiers.
\newblock {\em Machine Learning}, 6:161--182, 1991.

\bibitem{GYZ17}
S.~Guha, Y.~Li, and Q.~Zhang.
\newblock Distributed partial clustering.
\newblock In {\em SPAA}, pages 143--152, 2017.

\bibitem{Hanneke16}
S.~Hanneke.
\newblock The optimal sample complexity of pac learning.
\newblock {\em The Journal of Machine Learning Research}, 17(1):1319--1333,
  2016.

\bibitem{DPSV12a}
H.~D. III, J.~M. Phillips, A.~Saha, and S.~Venkatasubramanian.
\newblock Efficient protocols for distributed classification and optimization.
\newblock In {\em ALT}, pages 154--168, 2012.

\bibitem{DPSV12b}
H.~D. III, J.~M. Phillips, A.~Saha, and S.~Venkatasubramanian.
\newblock Protocols for learning classifiers on distributed data.
\newblock In {\em AISTATS}, pages 282--290, 2012.

\bibitem{LBKW14}
Y.~Liang, M.~Balcan, V.~Kanchanapally, and D.~P. Woodruff.
\newblock Improved distributed principal component analysis.
\newblock In {\em NIPS}, pages 3113--3121, 2014.

\bibitem{MMR08}
Y.~Mansour, M.~Mohri, and A.~Rostamizadeh.
\newblock Domain adaptation with multiple sources.
\newblock In {\em NIPS}, pages 1041--1048, 2008.

\bibitem{nguyen2018improved}
H.~L. Nguyen and L.~Zakynthinou.
\newblock {Improved Algorithms for Collaborative PAC Learning}.
\newblock {\em arXiv preprint arXiv:1805.08356}, 2018.

\bibitem{WKS16}
J.~Wang, M.~Kolar, and N.~Srebro.
\newblock Distributed multi-task learning.
\newblock In {\em AISTATS}, pages 751--760, 2016.

\end{thebibliography}
